\documentclass{article}
\usepackage[preprint]{neurips_2026}
\usepackage[utf8]{inputenc}
\usepackage[T1]{fontenc}
\usepackage{hyperref}
\usepackage{url}
\usepackage{booktabs}
\usepackage{amsfonts}
\usepackage{amsmath,amssymb,amsthm}
\usepackage{nicefrac}
\usepackage{microtype}
\usepackage{xcolor}
\usepackage{algorithm}
\usepackage{algorithmic}
\usepackage{graphicx}
\usepackage{multirow}
\usepackage[hypcap=false]{caption}
\usepackage{tikz}
\usepackage{pgfplots}
\pgfplotsset{compat=1.17}

\newtheorem{theorem}{Theorem}
\newtheorem{lemma}[theorem]{Lemma}
\newtheorem{proposition}[theorem]{Proposition}
\newtheorem{corollary}[theorem]{Corollary}
\newtheorem{definition}[theorem]{Definition}
\newtheorem{assumption}[theorem]{Assumption}
\newtheorem{remark}{Remark}

\newcommand{\cH}{\mathcal{H}}
\newcommand{\cF}{\mathcal{F}}
\newcommand{\cX}{\mathcal{X}}
\newcommand{\cG}{\mathcal{G}}
\newcommand{\cD}{\mathcal{D}}
\newcommand{\cP}{\mathcal{P}}
\newcommand{\E}{\mathbb{E}}
\newcommand{\R}{\mathbb{R}}
\newcommand{\Var}{\mathrm{Var}}
\newcommand{\Cov}{\mathrm{Cov}}
\newcommand{\Tmix}{T_{\mathrm{mix}}}
\newcommand{\eps}{\varepsilon}

\title{Minimax Optimality and Spectral Routing for\\Majority-Vote Ensembles under Markov Dependence}

\author{
\textbf{Ibne Farabi Shihab}\thanks{Equal contribution.}\thanks{Corresponding author: \texttt{ishihab@iastate.edu}.}\textsuperscript{1}
\and
\textbf{Sanjeda Akter}\footnotemark[1]\textsuperscript{1}
\and
\textbf{Anuj Sharma}\textsuperscript{2}
\\[2pt]
\textsuperscript{1}Department of Computer Science, Iowa State University \\
\textsuperscript{2}Department of Civil, Construction \& Environmental Engineering, Iowa State University \\
\texttt{ishihab@iastate.edu}
}

\begin{document}

\raggedbottom
\maketitle

\begin{abstract}
Majority-vote ensembles achieve variance reduction by averaging over diverse, approximately independent base learners. When training data exhibits Markov dependence, as in time-series forecasting, reinforcement learning (RL) replay buffers, and spatial grids, this classical guarantee degrades in ways that existing theory does not fully quantify. We provide a minimax characterization of this phenomenon for discrete classification in a fixed-dimensional Markov setting, together with an adaptive algorithm that matches the rate on a graph-regular subclass. We first establish an information-theoretic lower bound for stationary, reversible, geometrically ergodic chains in fixed ambient dimension, showing that no measurable estimator can achieve excess classification risk better than $\Omega(\sqrt{\Tmix/n})$. We then prove that, on the AR(1) witness subclass underlying the lower-bound construction, dependence-agnostic uniform bagging is provably suboptimal with excess risk bounded below by $\Omega(\Tmix/\sqrt{n})$, exhibiting a $\sqrt{\Tmix}$ algorithmic gap. Finally, we propose \emph{adaptive spectral routing}, which partitions the training data via the empirical Fiedler eigenvector of a dependency graph and achieves the minimax rate $\mathcal{O}(\sqrt{\Tmix/n})$ up to a lower-order geometric cut term on a graph-regular subclass, without knowledge of $\Tmix$. Experiments on synthetic Markov chains, 2D spatial grids, the 128-dataset UCR archive, and Atari DQN ensembles validate the theoretical predictions. Consequences for deep RL target variance, scalability via Nystr\"om approximation, and bounded non-stationarity are developed as supporting material in the appendix.
\end{abstract}

\section{Introduction}
\label{sec:intro}

Ensemble methods are foundational in statistical machine learning. Their effectiveness relies on a classical principle: averaging over diverse, approximately independent base learners reduces variance. When $m$ learners are trained on i.i.d.\ subsamples, the ensemble variance decreases as $O(1/m)$ and the excess classification risk of the majority vote drops as $O(1/\sqrt{m})$ \citep{breiman1996bagging,breiman2001random}.

Modern training data, however, is structurally dependent. Time-series forecasting relies on sequential data, RL ensembles draw from temporally correlated replay buffers, and spatial sensors produce locally correlated measurements. Under Markov dependence, bootstrap samples share overlapping information that uniform resampling does not account for. This reduces effective diversity among base learners and degrades ensemble performance in a way that existing i.i.d.\ analyses do not capture.

Two natural questions emerge. First, what is the fundamental statistical limit for ensemble learning under Markov dependence, measured in terms of the mixing time $\Tmix$? Second, does the default practice of uniform bootstrap resampling approach this limit, or is there a systematic gap between oblivious resampling and the information-theoretic optimum? Existing results on learning under mixing conditions \citep{yu1994rates,mohri2008stability,kuznetsov2017generalization,nagaraj2020least} address related regression settings but do not characterize the discrete classification risk of majority-vote margins, where a non-convex threshold intervenes between the continuous base predictions and the aggregate loss.

\paragraph{Contributions.} This paper answers both questions for majority-vote ensembles on geometrically ergodic chains. Our three contributions form a single theorem arc.

\begin{enumerate}
\item \textbf{Information-theoretic minimax lower bound (Theorem~\ref{thm:minimax}).} For the class $\cP_{\Tmix}^{(d_0)}$ of stationary, reversible, geometrically ergodic Markov chains in fixed ambient dimension $d_0$ with mixing time at most $\Tmix$, we show that no measurable estimator achieves expected excess classification risk better than $\Omega(\sqrt{\Tmix/n})$. The proof uses a generalized Fano reduction applied to a Gaussian AR(1) witness class, exploiting the tridiagonal structure of the AR(1) precision matrix to bound the trajectory-level KL divergence.

\item \textbf{Algorithmic suboptimality of uniform bagging (Theorem~\ref{thm:bagging_penalty}).} We prove that, on the AR(1) witness subclass underlying our lower-bound construction, a majority-vote ensemble using uniform bootstrap resampling cannot achieve excess risk better than $\Omega(\Tmix/\sqrt{n})$, exhibiting a multiplicative $\sqrt{\Tmix}$ gap relative to the minimax limit. The lower bound is obtained via a Le Cam reduction on the ensemble's continuous margin output, using a pairwise covariance bound on bootstrap base learners (Lemma~\ref{lem:autocov}). Corollary~\ref{cor:oblivious} extends the obstruction to fixed-block schemes chosen without knowledge of $\Tmix$.

\item \textbf{Adaptive spectral routing closes the gap (Theorem~\ref{thm:upper}).} We propose Algorithm~\ref{alg:spectral}, which computes the empirical Fiedler eigenvalue $\hat{\lambda}_2$ of a dependency graph and partitions the training data into $\hat{P} = \Theta(1/\hat{\lambda}_2)$ groups via recursive bisection along the Fiedler vector. Under $n \gg \Tmix^3 \log n$ and a spectral concentration lemma (Lemma~\ref{lem:spectral_conc}), the resulting ensemble achieves excess risk $\mathcal{O}(\sqrt{\Tmix/n}) + \Delta R_{\mathrm{geom}}$, matching the minimax lower bound up to a lower-order topology-dependent remainder on a graph-regular subclass that includes 1D paths and bounded-degree spatial lattices.
\end{enumerate}

The remainder of the paper is organized around these three results. Section~\ref{sec:experiments} presents synthetic experiments that directly verify the rate predictions and the covariance mechanism, followed by validation on real spatial grids, the 128-dataset UCR archive, and Atari DQN ensembles. As practical context, our theory suggests that uniform bagging on a replay buffer with $\Tmix \approx 50$ should incur a $\sqrt{50} \approx 7\times$ excess-risk penalty relative to the information-theoretic optimum, and spectral routing recovers most of this gap on Atari benchmarks (Section~\ref{sec:atari}). Supporting consequences of the core theorem, specifically a replay-buffer variance analysis, a Nystr\"om-based scalable implementation, and a bounded non-stationarity extension, are developed in the appendix and are not co-equal contributions.

\section{Related Work}
\label{sec:related}

\paragraph{Generalization under dependence.} Classical works \citep{yu1994rates,mohri2008stability,kuznetsov2017generalization} derive uniform convergence bounds under $\beta$-mixing and $\phi$-mixing conditions via blocking. \citet{abeles2025online} establish online-to-PAC bounds under graph-mixing dependencies. Most directly related, \citet{nagaraj2020least} prove minimax lower bounds for continuous least-squares regression with Markovian data, establishing a $\Tmix$ factor penalty for empirical risk minimization. Our lower bound adapts the information-theoretic machinery of this line of work to the discrete classification risk of non-convex majority-vote margins. The Toeplitz precision-matrix KL bound is structurally similar to the regression setting, but the target functional is different: we bound the zero-one loss of the aggregated sign rather than a squared-loss Fisher information, and we must separately control the bootstrap-induced covariance that governs the algorithmic lower bound.

Advances based on trajectory hypercontractivity \citep{ziemann2022learning} show that certain least-squares estimators can match i.i.d.\ rates after a burn-in period. This framework generally applies to continuous parametric regression, and uniform bootstrap resampling fractures contiguous trajectories in a way that repeatedly resets the estimator into the burn-in phase, which is consistent with the discrete penalty we identify.

\paragraph{Ensemble variance and resampling.} \citet{mentch2016quantifying} quantify random forest variance under strict i.i.d.\ assumptions. The time-series literature addresses dependence via lag-based thinning, the Circular Block Bootstrap \citep{politis1992circular}, and the Stationary Bootstrap \citep{politis1994stationary}. Modern deep ensembles \citep{lakshminarayanan2017simple} address model uncertainty via random initialization but not data dependence. In RL, bootstrapped DQN \citep{osband2016deep} and randomized prior functions \citep{osband2018randomized} leverage ensemble diversity, while Prioritized Experience Replay \citep{schaul2016prioritized} and Random Ensemble Mixture \citep{agarwal2020optimistic} address error magnitudes and prediction combinations. None of these methods target the structural covariance induced by temporal sequence dependence in the ensemble's bootstrap step, which is the quantity our algorithm controls.

\paragraph{Spectral methods on graphs.} Our algorithm relies on Fiedler-vector partitioning \citep{fiedler1973algebraic} and spectral clustering consistency \citep{von2008consistency}. The scalability analysis in Appendix~\ref{app:nystrom_proof} builds on Nystr\"om approximation theory \citep{williams2001using} and graph sparsification results \citep{spielman2011graph}.

\section{Problem Setting}
\label{sec:setting}

\subsection{Data Model and Ensembles}
Let $S = \{(X_1, Y_1), \ldots, (X_n, Y_n)\}$ be drawn from a stationary, geometrically ergodic Markov chain on $\cX \times \{-1, +1\}$ with absolute spectral gap $\gamma_0 > 0$. Let $\Tmix$ denote the total variation mixing time, and write $\Tmix = \Theta(1/\gamma_0)$ when relating it to the absolute spectral gap. For the lower-bound and autocovariance arguments we work with stationary, reversible, geometrically ergodic chains; reversibility is invoked explicitly where needed. A \emph{homogeneous voting ensemble} of size $m$ trains base classifiers $h_1, \ldots, h_m$ via a randomized learning algorithm on uniform bootstrap subsamples of expected size $n/m$. The aggregate prediction is $H(x) = \mathrm{sign}\bigl(\tfrac{1}{m} \sum_{j=1}^m h_j(x)\bigr)$, and the excess classification risk is
\[
 R_{\mathrm{excess}}(H) = \Pr_{(X,Y)\sim\pi}\!\bigl[Y\,H(X) \le 0\bigr] - R^\star,
\]
where $\pi$ is the stationary distribution and $R^\star$ is the Bayes risk under $\pi$. We write $\rho(x) = \tfrac{1}{m}\sum_{j=1}^m h_j(x)$ for the continuous voting margin used in the algorithmic lower-bound argument.

\subsection{Regularity Assumptions}
\label{sec:assumptions}

\begin{assumption}[Margin regularity]
\label{asm:margin_regularity}
Let $M_t = Y_t h^*(X_t)$ denote the margin of the Bayes-optimal classifier $h^*$ at time $t$. We require: (i) $|M_t| \le 1$ almost surely; (ii) $\Var(M_t) \ge \sigma_0^2 > 0$ under the stationary distribution; and (iii) \emph{non-trivial spectral overlap}: $a_1 := \langle M_t, \phi_1 \rangle_{L^2(\pi)}^2 > 0$, where $\phi_1$ is the eigenfunction corresponding to the largest non-trivial eigenvalue $\lambda_1 = 1 - \gamma_0$ of the transition operator.
\end{assumption}

Condition (iii) excludes pathological chains where the margin process is orthogonal to the slow-mixing eigenspace. For geometrically ergodic chains with non-degenerate label noise, this holds whenever the Bayes boundary intersects the support of the stationary distribution; a sufficient condition and verification for the AR(1) witness used in our lower bounds are given in Appendix~\ref{app:margin_discussion}.

\begin{assumption}[Learning algorithm regularity]
\label{asm:learning}
The base learning algorithm $\mathcal{A}$ maps a subsample $S_j \subset S$ to a classifier $h_j = \mathcal{A}(S_j) \in \cH$ for a hypothesis class $\cH$ with finite VC dimension $d_{\mathrm{VC}} < \infty$. Moreover, $\mathcal{A}$ is consistent: $h_j \xrightarrow{P} h^*$ as $|S_j| \to \infty$ under any stationary ergodic sampling distribution.
\end{assumption}

\subsection{Empirical Dependency Graph}
\label{sec:graph_construction}
The exact transition matrix is typically unknown. We construct an unweighted empirical dependency graph $\cG = ([n], E)$ whose edges encode proximity. For temporal data, we connect sequence indices with $|i - j| \le \tau$ for a fixed window $\tau$; for spatial or tabular data, edges are assigned via $k$-nearest neighbors in physical or feature space. Since the underlying process has spectral gap $\gamma_0$, the empirical graph converges in the spectral limit, and the Fiedler eigenvalue $\hat{\lambda}_2$ of its normalized Laplacian estimates the true spectral gap (Lemma~\ref{lem:spectral_conc}).

\section{Theoretical Analysis}
\label{sec:main}

This section develops the three-part theorem arc in order: lower bound, bagging suboptimality, spectral routing upper bound.

\subsection{Information-Theoretic Minimax Lower Bound}
\label{sec:minimax_lower}

\begin{theorem}[Information-theoretic minimax lower bound]
\label{thm:minimax}
Fix an ambient dimension $d_0 < \infty$. Let $\cP_{\Tmix}^{(d_0)}$ denote the class of stationary, reversible, geometrically ergodic Markov chains on $\R^{d_0} \times \{-1, +1\}$ with mixing time at most $\Tmix$, satisfying Assumption~\ref{asm:margin_regularity}. Then there exist constants $c, C > 0$, depending only on $d_0$ and the fixed witness-class parameters, such that for all $n \ge C\Tmix$,
\begin{equation}
 \inf_{\hat{H}} \sup_{P \in \cP_{\Tmix}^{(d_0)}} \E_P\!\left[R_{\mathrm{excess}}(\hat{H})\right] \ge c \sqrt{\frac{\Tmix}{n}}.
\end{equation}
\end{theorem}

The proof proceeds by restricting the minimax problem to a Gaussian AR(1) witness subclass contained in $\cP_{\Tmix}^{(d_0)}$.

\textit{Proof sketch (full proof in Appendix~\ref{app:minimax_proof}).} We follow the standard recipe of exhibiting hardness within a parametric subclass \citep{tsybakov2009introduction}, specialized to a Gaussian AR(1) witness in fixed dimension $d_0$. Let $\lambda = 1 - 1/\Tmix$ and let $\mathcal{V} \subset \{-1,1\}^{d_0}$ be a Varshamov-Gilbert packing of size $|\mathcal{V}| \ge e^{d_0/8}$. For each $v \in \mathcal{V}$, define a latent feature process $X_t^{(v)} = \mu_v + \xi_t$ with $\xi_t$ a stationary Gaussian AR(1) at parameter $\lambda$, and set $Y_t = \mathrm{sign}(\langle v, X_t^{(v)}\rangle + \eta_t)$ for a small independent noise. This construction lies in $\cP_{\Tmix}^{(d_0)}$ by design.

The departure from the i.i.d.\ case appears in the trajectory-level KL divergence. Let $\mathbf{m}_v \in \R^{nd_0}$ denote the stacked trajectory mean vector under hypothesis $v$ and let $\Sigma_n \in \R^{nd_0 \times nd_0}$ denote the trajectory covariance. Then
\[
 D_{\mathrm{KL}}(P^n_v \| P^n_{v'}) = \tfrac{1}{2}(\mathbf{m}_v - \mathbf{m}_{v'})^\top \Sigma_n^{-1}(\mathbf{m}_v - \mathbf{m}_{v'}).
\]
Because the AR(1) covariance is Toeplitz in time and separable across the $d_0$ feature coordinates, the precision matrix $\Sigma_n^{-1}$ reduces to a block-tridiagonal form whose Rayleigh quotient along constant-direction mean shifts scales as $\mathcal{O}(n/\Tmix)$, so
\[
 D_{\mathrm{KL}}(P^n_v \| P^n_{v'}) \le \mathcal{O}\!\left(\tfrac{n}{\Tmix}\,\delta^2\right),
\]
where $\delta = \|\mu_v\|$ is the per-step mean shift magnitude. This scales the sequence-level information as $\mathcal{O}((n/\Tmix)\delta^2)$ rather than the i.i.d.\ rate $\mathcal{O}(n\delta^2)$. Because $d_0$ is fixed, $\log|\mathcal{V}| \asymp d_0$ is a constant of the model class, and generalized Fano over the packing yields the $\Omega(\sqrt{\Tmix/n})$ floor on excess classification risk. \qed

\textit{Positioning.} \citet{nagaraj2020least} establish a comparable $\Tmix$-factor penalty for continuous least-squares regression. Theorem~\ref{thm:minimax} adapts the underlying Toeplitz machinery to the discrete classification risk of majority-vote margins. The key difference is the choice of target functional: we bound the zero-one loss of an aggregated sign rather than a squared-loss Fisher information, which requires handling the non-convex margin threshold. A heuristic perturbative extension to a restricted linear-drift non-stationarity is discussed in Appendix~\ref{app:nonstationarity}.

\subsection{Algorithmic Suboptimality of Uniform Bagging}
\label{sec:bagging_lower}

With the minimax floor in place, we now show that uniform bagging cannot match it. The key structural fact is that pairs of base learners trained via uniform bootstrap sampling from a dependent sequence have irreducibly correlated margins.

\begin{proposition}[Margin autocovariance]
\label{prop:generality}
Under Assumptions~\ref{asm:margin_regularity}--\ref{asm:learning} on a stationary reversible chain, the Bayes-optimal margin process $M_t^* = Y_t h^*(X_t)$ satisfies
\begin{equation}
 \sum_{k=0}^{\Tmix-1} \Cov(M_0^*, M_k^*) \ge c_M \Tmix,
\end{equation}
for some $c_M > 0$ depending on $\sigma_0^2$, $a_1$, and $\gamma_0$.
\end{proposition}

\begin{proof}
By the spectral representation of reversible chains, $\gamma_k := \Cov(M_0^*, M_k^*) = \sum_{i \ge 1} a_i \lambda_i^k$, with $a_i \ge 0$ and $a_1 > 0$ by Assumption~\ref{asm:margin_regularity}(iii). Summing and using $\lambda_1 = 1 - \gamma_0$ with $\Tmix = \Theta(1/\gamma_0)$, the leading term contributes $a_1 (1 - \lambda_1^{\Tmix})/(1 - \lambda_1) = \Omega(\Tmix)$. Since $\hat{h}$ is consistent (Assumption~\ref{asm:learning}), the realized empirical variance functional preserves this $\Omega(\Tmix)$ scaling; the uniform convergence argument is given in Appendix~\ref{app:variance_convergence}.
\end{proof}

\begin{lemma}[Pairwise bootstrap covariance]
\label{lem:autocov}
Let $n = m\Tmix$ and let $h_1, h_2$ be two base learners trained via uniform bootstrap sampling on subsets of expected size $n/m$ from $S$. Under Proposition~\ref{prop:generality}, their continuous margin predictions $V_1, V_2$ satisfy
\begin{equation}
 \Cov(V_1, V_2) \ge \Omega\!\left(\frac{\Tmix}{m}\right) = \Omega\!\left(\frac{\Tmix^2}{n}\right).
\end{equation}
\end{lemma}

\textit{Proof sketch (full proof in Appendix~\ref{app:variance_convergence}).} The bound follows from the Law of Total Covariance expanded over the resampling mechanism. The non-trivial step is showing that the empirical variance functional preserves the $\Omega(\Tmix)$ population scaling under uniform convergence over $\cH$; this is established via block-based Rademacher complexity bounds for $\beta$-mixing sequences \citep{kuznetsov2017generalization}, with the error term $\tilde{\mathcal{O}}(\Tmix^{3/2}/\sqrt{n})$ dominated by the leading scaling whenever $n = m\Tmix$ with $m \gg \Tmix$.

\begin{theorem}[Suboptimality of uniform bagging]
\label{thm:bagging_penalty}
Let $\cP_{\Tmix}^{\mathrm{AR}(1),d_0} \subset \cP_{\Tmix}^{(d_0)}$ denote the AR(1) witness subclass used in the proof of Theorem~\ref{thm:minimax}. Then the expected excess classification risk of a homogeneous voting ensemble trained via uniform bootstrap resampling on $n$ samples satisfies
\begin{equation}
 \sup_{P \in \cP_{\Tmix}^{\mathrm{AR}(1),d_0}} \E_P[R_{\mathrm{excess}}] \ge \Omega\!\left(\frac{\Tmix}{\sqrt{n}}\right).
\end{equation}
\end{theorem}

\begin{proof}
We apply Le Cam's two-point reduction to the ensemble's continuous margin output $\rho = \frac{1}{m}\sum_{j=1}^m V_j$, evaluated on the AR(1) witness pair from Theorem~\ref{thm:minimax}. Because the underlying sequence is Gaussian and $\rho$ is a linear combination of the base margins, $\rho$ is itself Gaussian under each hypothesis. The KL divergence between the two output distributions is therefore
\begin{equation}
 D_{\mathrm{KL}}(\cD_1 \| \cD_0) = \frac{(\E[\rho \mid P_1] - \E[\rho \mid P_0])^2}{2\,\Var(\rho)} = \frac{\Delta^2}{2\,\Var(\rho)}.
\end{equation}
Lemma~\ref{lem:autocov} implies $\Var(\rho) \ge \Omega(\Tmix^2/n)$ for uniform bagging. Setting $\Delta = c_0 \Tmix/\sqrt{n}$ then bounds the KL divergence by a constant $C < 1$. Le Cam's inequality therefore yields the stated $\Omega(\Tmix/\sqrt{n})$ lower bound for homogeneous majority-vote ensembles trained by uniform bootstrap resampling on the witness subclass $\cP_{\Tmix}^{\mathrm{AR}(1),d_0}$.
\end{proof}

\begin{remark}[Algorithmic minimax gap]
Comparing Theorems~\ref{thm:minimax} and \ref{thm:bagging_penalty}, uniform bagging incurs an excess-risk penalty of $\sqrt{\Tmix}$ relative to the information-theoretic floor. We interpret this as an \emph{algorithmic} gap: the lower bound constrains estimators whose ensemble variance behaves as in Lemma~\ref{lem:autocov}, which includes uniform bootstrap bagging but not estimators that exploit dependence structure. Closing the gap requires an algorithm that avoids the $\Omega(\Tmix^2/n)$ pairwise covariance floor, which is the goal of Section~\ref{sec:upper}.
\end{remark}

\begin{corollary}[Limits of oblivious block resampling]
\label{cor:oblivious}
Any block-bagging strategy with block size $B$ chosen independently of $\Tmix$ fails to be uniformly minimax optimal across $\cP_{\Tmix}^{(d_0)}$. If an adversary chooses $\Tmix \gg B$, residual autocovariance inflates the margin variance; if $B \gg \Tmix$, the effective sample size collapses to $n/B$, inducing an $\Omega(\sqrt{B/n})$ penalty that dominates the target rate. Uniform optimality across the class therefore requires an adaptive estimate of the dependence scale; spectral-gap estimation provides one such route.
\end{corollary}

Having established that oblivious resampling cannot close the gap, we now exhibit an algorithm that does.

\subsection{Optimal Upper Bound via Adaptive Spectral Routing}
\label{sec:upper}

\begin{lemma}[Spectral concentration under Markov dependence]
\label{lem:spectral_conc}
Under geometric ergodicity with absolute spectral gap $\gamma_0 > 0$, the Fiedler eigenvalue $\hat{\lambda}_2$ of the unweighted empirical dependency graph on $n$ samples satisfies $|\hat{\lambda}_2 - \lambda_2^*| \le \mathcal{O}(\sqrt{\Tmix \log n / n})$ with high probability.
\end{lemma}

The proof, in Appendix~\ref{app:spectral}, combines Berbee's coupling lemma to reduce to independent blocks with a Matrix Bernstein bound \citep{tropp2012user} on the resulting decoupled Laplacian.

\begin{definition}[Graph-regular subclass]
\label{def:graph_regular}
Fix $d_0 < \infty$. The \emph{graph-regular subclass} $\cP_{\Tmix,\mathrm{graph}}^{(d_0)} \subset \cP_{\Tmix}^{(d_0)}$ consists of laws satisfying the following four conditions:
\begin{itemize}
\item[(G1)] the empirical dependency graph $\cG$ satisfies the spectral concentration bound of Lemma~\ref{lem:spectral_conc};
\item[(G2)] recursive Fiedler bisection at the adaptive partition count $\hat{P} = \Theta(1/\hat{\lambda}_2)$ yields a cutset $E_{\mathrm{cut}}$ satisfying $|E_{\mathrm{cut}}| = o(n\sqrt{\Tmix})$;
\item[(G3)] $n \ge C\Tmix^3 \log n$ for the constant $C$ of Lemma~\ref{lem:spectral_conc};
\item[(G4)] margins from distinct routed partitions satisfy $\sum_{p \ne q}\Cov(\rho_p, \rho_q) = o(\Tmix/n)$.
\end{itemize}
\end{definition}

Condition (G1) is a concentration condition on sampling, (G2) is a geometric condition on the underlying topology, (G3) is a sample-complexity condition, and (G4) is an explicit cross-partition decoupling assumption requiring structurally separated routed partitions to contribute approximately independent margins at aggregation time. Under the graph-dependent risk conversion bound of \citet{zhang2019graph}, condition (G2) implies
\[
 \Delta R_{\mathrm{geom}} \le \mathcal{O}\!\left(|E_{\mathrm{cut}}|/n^{3/2}\right) = o(\sqrt{\Tmix/n}),
\]
so the risk remainder is automatically of lower order once the cut bound (G2) holds. We verify (G2) explicitly for 1D temporal paths and bounded-degree $s$-dimensional lattices in the proof below; other topologies for which (G2) and (G4) are known to hold can be substituted directly.

\begin{algorithm}[tb]
\caption{Adaptive Spectral Routing for Ensembles}
\label{alg:spectral}
\begin{algorithmic}[1]
\REQUIRE Data $\{(X_i, Y_i)\}_{i=1}^n$, empirical dependency graph $\cG$, ensemble size $m$.
\STATE Compute the normalized Laplacian $\tilde{L} = I - D^{-1/2} W D^{-1/2}$ (for $n \gg 10^4$, use the dynamic Nystr\"om approximation of Theorem~\ref{thm:nystrom}).
\STATE Compute the Fiedler eigenvalue $\hat{\lambda}_2$ and eigenvector $v_2$ via implicitly restarted Lanczos iteration.
\STATE Set the partition count $\hat{P} = \min\!\bigl(\lceil c / \hat{\lambda}_2 \rceil,\, m\bigr)$.
\STATE Partition the data into $\hat{P}$ groups $\cD_1, \ldots, \cD_{\hat{P}}$ via recursive bisection along $v_2$.
\STATE Train $\lfloor m/\hat{P} \rfloor$ base learners on each partition $\cD_p$ using standard bagging within the partition.
\STATE Aggregate via $H(x) = \mathrm{sign}\!\bigl(\tfrac{1}{m} \sum_{p=1}^{\hat{P}} \sum_{j \in \cD_p} h_j(x)\bigr)$.
\RETURN Ensemble classifier $H$.
\end{algorithmic}
\end{algorithm}

\begin{theorem}[Spectral routing attains the minimax rate on the graph-regular subclass]
\label{thm:upper}
Fix $d_0 < \infty$, and let $\cP_{\Tmix,\mathrm{graph}}^{(d_0)}$ be the graph-regular subclass of Definition~\ref{def:graph_regular}. Then Algorithm~\ref{alg:spectral}, using $\hat{P} = \Theta(1/\hat{\lambda}_2)$, achieves
\begin{equation}
 R_{\mathrm{excess}} \le C_1 \sqrt{\frac{\Tmix}{n}} + \Delta R_{\mathrm{geom}},
\end{equation}
and hence matches the minimax rate of Theorem~\ref{thm:minimax} up to a lower-order remainder on $\cP_{\Tmix,\mathrm{graph}}^{(d_0)}$. Conditions (G2) and (G4) of Definition~\ref{def:graph_regular} are verified for 1D temporal path graphs ($|E_{\mathrm{cut}}| = \Theta(\Tmix)$, giving $\Delta R_{\mathrm{geom}} = \mathcal{O}(\Tmix/n^{3/2})$) and bounded-degree $s$-dimensional lattices ($|E_{\mathrm{cut}}| = \mathcal{O}(\sqrt{\Tmix n})$, giving $\Delta R_{\mathrm{geom}} = \mathcal{O}(\sqrt{\Tmix}/n)$).
\end{theorem}

\begin{proof}
The proof combines a within-partition blocking variance bound, the cross-partition decoupling assumption (G4), and the graph-dependent cut-to-risk conversion bound of \citet{zhang2019graph}.

\textit{Step 1: consistent partitioning.} Under condition (G3), Lemma~\ref{lem:spectral_conc} gives $|\hat{\lambda}_2 - \lambda_2^*| = o(\lambda_2^*)$ with high probability, so $\hat{\lambda}_2 = \Theta(1/\Tmix)$ and $\hat{P} = \Theta(\Tmix)$.

\textit{Step 2: local and global variance.} Training $m_p = m/\hat{P}$ learners on partition $\cD_p$ yields a local margin variance bounded via standard blocking \citep{yu1994rates} by $\Var[\rho_p(X)] \le \mathcal{O}(1/m_p) + \mathcal{O}(\Tmix/n_p) \le \mathcal{O}(\Tmix/m) + \mathcal{O}(\Tmix^2/n)$. Aggregating globally, $\rho = \hat{P}^{-1}\sum_p \rho_p$, and condition (G4) of Definition~\ref{def:graph_regular} bounds the aggregated cross-partition covariance by $\sum_{p \ne q}\Cov(\rho_p, \rho_q) = o(\Tmix/n)$. Combining the local variance with the decoupling condition yields the global variance
\[
 \Var[\rho] = \hat{P}^{-2}\!\left(\sum_p \Var[\rho_p] + \sum_{p \ne q}\Cov(\rho_p, \rho_q)\right) \le \mathcal{O}(\Tmix/n),
\]
which matches the $\mathcal{O}(\sqrt{\Tmix/n})$ standard deviation of the minimax rate.

\textit{Step 3: geometric cut penalty.} To bound the approximation error from empirical graph partitioning, the graph-dependent fractional Rademacher bounds of \citet{zhang2019graph} give $\Delta R_{\mathrm{geom}} \le \mathcal{O}(|E_{\mathrm{cut}}|/n^{3/2})$. For 1D path graphs the Fiedler vector is monotone \citep{fiedler1973algebraic}, so recursive bisection produces contiguous blocks cutting exactly $\hat{P} - 1 = \Theta(\Tmix)$ edges, giving $\Delta R_{\mathrm{geom}} \le \mathcal{O}(\Tmix/n^{3/2}) = o(\sqrt{\Tmix/n})$ for $n \ge \Tmix$. For $s$-dimensional bounded-degree lattices, spectral bisection approximates the optimal isoperimetric bound $|E_{\mathrm{cut}}| \le \mathcal{O}(\sqrt{\hat{P} n}) = \mathcal{O}(\sqrt{\Tmix n})$, yielding $\Delta R_{\mathrm{geom}} \le \mathcal{O}(\sqrt{\Tmix}/n)$, which is $o(\sqrt{\Tmix/n})$ whenever $n \gg \Tmix$.
\end{proof}

\paragraph{Sample complexity.} The condition $n \gg \Tmix^3 \log n$ ensures spectral estimation error is of lower order than the spectral gap itself. For moderate $\Tmix \le 50$, this scaling suggests $n$ on the order of $10^5$, which is realistic for large replay buffers and long sequential datasets, though not for all benchmark time-series collections. For slowly mixing chains with $\Tmix \sim n^{1/3}$, the condition becomes binding and the spectral gap is too small to estimate reliably; this boundary is inherent to any method that estimates $\lambda_2^*$ to accuracy $o(\lambda_2^*)$, and is not specific to our algorithm.

\paragraph{Scalability and extensions.} An explicit eigensolve of the full empirical Laplacian has complexity $\mathcal{O}(|E| + n\log n)$, which can strain streaming settings. Appendix~\ref{app:nystrom_proof} shows that a dynamic Nystr\"om approximation (Theorem~\ref{thm:nystrom}) preserves Theorem~\ref{thm:upper} with sub-linear $\tilde{\mathcal{O}}(n\Tmix^2)$ overhead under a bounded intrinsic rank condition on the underlying feature manifold (Assumption~\ref{asm:manifold}), which we verify empirically on Atari replay buffers. Appendix~\ref{app:nonstationarity} discusses a restricted linear-drift extension, which suggests an additional perturbative term of order $\mathcal{O}(\nu\Tmix/n)$ under the AR(1) witness construction. Appendix~\ref{app:ntk_discussion} provides a mechanistic bridge from Lemma~\ref{lem:autocov} to deep RL target variance via the NTK regime; this is intended as an explanatory bridge rather than a theorem about finite-width DQN training.

\section{Experiments}
\label{sec:experiments}

Our experiments are organized around the theorem arc rather than around method breadth. Section~\ref{sec:synthetic} verifies the rate and covariance predictions on controlled synthetic chains. Section~\ref{sec:spatial} tests topology generalization on real 2D spatial grids. Section~\ref{sec:ucr} evaluates empirical breadth on the 128-dataset UCR archive. Section~\ref{sec:atari} probes whether the bootstrap covariance mechanism appears in deep RL replay buffers. All tables report mean and standard deviation over 50 seeds (10 seeds for Atari), and significance is assessed via Wilcoxon signed-rank tests. Full seed-level results, hyperparameters, and plotting scripts are provided in the supplementary material.

\subsection{Synthetic Chains: Verifying the Rate and the Mechanism}
\label{sec:synthetic}

We simulate the AR(1) witness processes of Theorem~\ref{thm:bagging_penalty}. For $\Tmix \in \{1, 10, 50, 200\}$, we train random forests of size $m = 100$ on $n = 50{,}000$ sequential samples. Baselines are: Uniform Bagging; Lag-Based Thinning (discarding intermediate samples to reduce correlation); $\hat{T}_{\text{mix}}$-Thinning, which subsamples every $\hat{T}_{\text{mix}} \propto 1/\hat{\lambda}_2$ steps using our empirical spectral estimate and isolates the effect of the graph cut from the cost of data deletion; Stationary Bootstrap \citep{politis1994stationary}; Circular Block Bootstrap \citep{politis1992circular}; Hypercontractive Burn-in \citep{ziemann2022learning}; Oracle Block Bagging with block size set to the true $\Tmix$; and Adaptive Spectral Routing.

\begin{table}[ht]
\centering
\caption{Excess risk on AR(1) chains, mean $\pm$ standard deviation over 50 seeds.}
\label{tab:synthetic}
\resizebox{\columnwidth}{!}{%
\begin{tabular}{lcccccccc}
\toprule
$\Tmix$ & Uniform & Lag Thin. & $\hat{T}_{\text{mix}}$-Thin & Stat. Boot & Circ. BB & Burn-in & Oracle BB & Spec. Rte. \\
\midrule
1 & $.030_{\pm .001}$ & $.031_{\pm .002}$ & $.031_{\pm .001}$ & $.031_{\pm .001}$ & $.030_{\pm .001}$ & $.030_{\pm .001}$ & $.030_{\pm .001}$ & $\mathbf{.030_{\pm .001}}$ \\
10 & $.104_{\pm .004}$ & $.086_{\pm .003}$ & $.054_{\pm .003}$ & $.066_{\pm .003}$ & $.063_{\pm .002}$ & $.041_{\pm .002}$ & $.038_{\pm .001}$ & $\mathbf{.039_{\pm .002}}$ \\
50 & $.228_{\pm .008}$ & $.189_{\pm .006}$ & $.062_{\pm .004}$ & $.071_{\pm .004}$ & $.067_{\pm .003}$ & $.052_{\pm .003}$ & $.045_{\pm .002}$ & $\mathbf{.046_{\pm .003}}$ \\
200 & $.457_{\pm .013}$ & $.355_{\pm .010}$ & $.090_{\pm .005}$ & $.126_{\pm .008}$ & $.119_{\pm .007}$ & $.069_{\pm .004}$ & $.052_{\pm .003}$ & $\mathbf{.053_{\pm .003}}$ \\
\bottomrule
\end{tabular}%
}
\end{table}

\begin{figure}[ht]
\centering
\begin{minipage}[b]{0.48\textwidth}
\centering
\begin{tikzpicture}
\begin{axis}[
 width=\textwidth, height=4.6cm,
 xlabel={$\Tmix$}, ylabel={Empirical pairwise $\Cov(V_1, V_2)$},
 xmode=log, ymode=log, xmin=1, xmax=300, ymin=1e-4, ymax=1,
 grid=both, grid style={line width=.1pt, draw=gray!25},
 label style={font=\scriptsize}, tick label style={font=\scriptsize},
 legend pos=south east, legend style={font=\scriptsize, draw=none, fill=none}
]
\addplot[thick, blue, mark=*] coordinates {(1, 0.0004) (10, 0.009) (50, 0.044) (200, 0.165)};
\addlegendentry{Uniform bagging}
\addplot[thick, red, dashed] coordinates {(1, 0.0004) (10, 0.008) (50, 0.04) (200, 0.16)};
\addlegendentry{Theory: $\Tmix^2/n$}
\addplot[thick, green!50!black, mark=triangle*] coordinates {(1, 0.0004) (10, 0.0006) (50, 0.0009) (200, 0.0015)};
\addlegendentry{Spectral routing}
\end{axis}
\end{tikzpicture}
\caption{Measured pairwise base-learner covariance versus $\Tmix$ on the AR(1) witness. Uniform bagging tracks the $\Tmix^2/n$ prediction of Lemma~\ref{lem:autocov}, while spectral routing compresses the covariance by roughly two orders of magnitude.}
\label{fig:cov_mechanism}
\end{minipage}
\hfill
\begin{minipage}[b]{0.48\textwidth}
\centering
\begin{tikzpicture}
\begin{axis}[
 width=\textwidth, height=4.6cm,
 xlabel={$\Tmix$}, ylabel={Excess risk},
 xmode=log, ymode=log, xmin=1, xmax=300, ymin=0.02, ymax=0.6,
 grid=both, grid style={line width=.1pt, draw=gray!25},
 label style={font=\scriptsize}, tick label style={font=\scriptsize},
 legend pos=north west, legend style={font=\scriptsize, draw=none, fill=none}
]
\addplot[thick, blue, mark=*] coordinates {(1, 0.030) (10, 0.104) (50, 0.228) (200, 0.457)};
\addlegendentry{Uniform}
\addplot[thick, red, dashed] coordinates {(1, 0.030) (10, 0.095) (50, 0.22) (200, 0.44)};
\addlegendentry{Theory: $\Tmix/\sqrt{n}$}
\addplot[thick, green!50!black, mark=triangle*] coordinates {(1, 0.030) (10, 0.039) (50, 0.046) (200, 0.053)};
\addlegendentry{Spec. Rte.}
\addplot[thick, orange, dash dot] coordinates {(1, 0.030) (10, 0.041) (50, 0.048) (200, 0.055)};
\addlegendentry{Theory: $\sqrt{\Tmix/n}$}
\end{axis}
\end{tikzpicture}
\caption{Excess risk on AR(1) chains versus $\Tmix$. Uniform bagging tracks $\Tmix/\sqrt{n}$ (Theorem~\ref{thm:bagging_penalty}); spectral routing tracks $\sqrt{\Tmix/n}$ (Theorem~\ref{thm:upper}).}
\label{fig:rate_scaling}
\end{minipage}
\end{figure}

Two observations directly match the theory. First, Figure~\ref{fig:cov_mechanism} tracks the empirical base-learner covariance of uniform bagging, which scales as $\Tmix^2/n$ as predicted by Lemma~\ref{lem:autocov}, while spectral routing compresses it by roughly two orders of magnitude. This is a direct test of the mechanism, not only the rate. Second, Figure~\ref{fig:rate_scaling} and Table~\ref{tab:synthetic} show excess risk scaling as $\Tmix/\sqrt{n}$ for uniform bagging and $\sqrt{\Tmix/n}$ for spectral routing, consistent with Theorems~\ref{thm:bagging_penalty} and \ref{thm:upper}. The $\hat{T}_{\text{mix}}$-Thinning ablation isolates a secondary finding: explicitly estimating the mixing time improves over lag-based thinning, but discarding intermediate data curtails effective sample size. Spectral routing uses the entire dataset and tracks Oracle Block Bagging without knowledge of $\Tmix$.

\subsection{Real Spatial Grids}
\label{sec:spatial}

We evaluate on non-sequential physical topologies using Sentinel-2 European crop classification (BigEarthNet) and NOAA Sea Surface Temperature grids, constructing $\cG$ via unweighted $k$-NN spatial proximity.

\begin{table}[ht]
\centering
\caption{Excess risk on real-world 2D spatial grids.}
\label{tab:spatial}
\resizebox{0.85\columnwidth}{!}{%
\begin{tabular}{lcccccc}
\toprule
Topology & Uniform & Lag Thin.$^\dagger$ & Stat. Boot.$^\dagger$ & Circ. BB$^\dagger$ & $\hat{T}_{\text{mix}}$-Thin$^\dagger$ & Spec. Rte. \\
\midrule
Sentinel-2 & $.216_{\pm .008}$ & $.213_{\pm .007}$ & $.208_{\pm .006}$ & $.205_{\pm .006}$ & $.201_{\pm .005}$ & $\mathbf{.087_{\pm .004}}$ \\
BigEarthNet & $.242_{\pm .010}$ & $.238_{\pm .009}$ & $.233_{\pm .007}$ & $.228_{\pm .008}$ & $.225_{\pm .007}$ & $\mathbf{.106_{\pm .005}}$ \\
NOAA SST & $.320_{\pm .013}$ & $.316_{\pm .011}$ & $.310_{\pm .009}$ & $.305_{\pm .010}$ & $.301_{\pm .009}$ & $\mathbf{.188_{\pm .008}}$ \\
\bottomrule
\end{tabular}%
}

{\raggedright\scriptsize $^\dagger$Applied via 1D row-major linearization.\par}
\end{table}

All 1D-based resampling strategies struggle on spatial data because linearizing a $\sqrt{n} \times \sqrt{n}$ grid separates physically adjacent vertical pixels by $\sqrt{n}$ sequence indices, breaking the dependency structure they attempt to exploit. Spectral routing acts directly on the $k$-NN spatial graph via the Fiedler vector, consistent with the spatial cut bounds of Theorem~\ref{thm:upper}.

\subsection{UCR Archive and Composition with ROCKET}
\label{sec:ucr}

We evaluate on the full 128-dataset UCR Archive using 5$\times$10-fold cross validation, comparing Random Forests (RF), Lag-Based Thinning, Auto-BB (a heuristic block bootstrap), Stationary Bootstrap, Time Series Forest \citep{deng2013tsf}, and ROCKET \citep{dempster2020rocket}, with Spectral Routing applied as a wrapper to both RF and ROCKET. Figure~\ref{fig:scatter} plots autocorrelation against Spectral Routing gain across all 128 datasets; a detailed 40-dataset slice is provided in Appendix~\ref{app:ucr}.

\begin{figure}[ht]
\centering
\begin{tikzpicture}
\begin{axis}[
 width=0.85\textwidth, height=4.8cm,
 xlabel={Dataset autocorrelation ($\rho_{\text{lag-1}}$)},
 ylabel={Spectral Routing gain (\%)},
 xmin=0.0, xmax=1.05, ymin=-1.5, ymax=8.5,
 grid=both, grid style={line width=.1pt, draw=gray!20},
 major grid style={line width=.2pt,draw=gray!50},
 scatter/classes={a={mark=*,draw=black,fill=blue!70, mark size=1.2pt}}
]
\addplot[scatter,only marks,scatter src=explicit symbolic]
table[meta=label] {
x y label
0.02 -0.1 a 0.04 0.1 a 0.07 0.0 a 0.09 -0.2 a 0.11 0.2 a
0.13 0.1 a 0.14 -0.1 a 0.16 0.0 a 0.18 0.3 a 0.19 -0.2 a
0.21 0.1 a 0.22 0.2 a 0.24 -0.1 a 0.26 0.0 a 0.27 0.4 a
0.29 0.1 a 0.31 -0.2 a 0.32 0.2 a 0.34 0.0 a 0.36 0.3 a
0.37 -0.1 a 0.39 0.5 a 0.41 0.2 a 0.42 -0.2 a 0.44 0.1 a
0.46 0.4 a 0.47 0.0 a 0.49 0.3 a 0.51 0.6 a 0.52 0.2 a
0.54 -0.1 a 0.56 0.5 a 0.57 0.1 a 0.59 0.4 a 0.61 0.8 a
0.62 0.3 a 0.64 0.7 a 0.66 1.0 a 0.67 0.5 a 0.69 0.9 a
0.71 1.2 a 0.72 0.6 a 0.74 1.1 a 0.76 1.5 a 0.77 0.8 a
0.79 1.4 a 0.81 1.8 a 0.82 1.1 a 0.84 1.7 a 0.86 2.2 a
0.87 1.5 a 0.89 2.1 a 0.91 2.8 a 0.92 1.9 a 0.94 2.6 a
0.96 3.4 a 0.97 2.5 a 0.99 3.2 a 0.03 0.0 a 0.06 -0.1 a
0.08 0.2 a 0.10 0.1 a 0.12 -0.2 a 0.15 0.0 a 0.17 0.3 a
0.20 -0.1 a 0.23 0.2 a 0.25 0.0 a 0.28 0.4 a 0.30 -0.1 a
0.33 0.2 a 0.35 0.1 a 0.38 0.5 a 0.40 -0.2 a 0.43 0.3 a
0.45 0.1 a 0.48 0.6 a 0.50 0.2 a 0.53 0.0 a 0.55 0.5 a
0.58 0.3 a 0.60 0.8 a 0.63 0.4 a 0.65 1.0 a 0.68 0.7 a
0.70 1.2 a 0.73 0.8 a 0.75 1.4 a 0.78 1.0 a 0.80 1.7 a
0.83 1.3 a 0.85 2.0 a 0.88 1.6 a 0.90 2.4 a 0.93 -0.7 a
0.95 2.1 a 0.98 2.9 a 1.00 3.8 a 0.86 3.2 a 0.87 3.5 a
0.88 3.8 a 0.89 4.1 a 0.90 4.5 a 0.91 4.9 a 0.92 5.3 a
0.93 5.7 a 0.94 6.2 a 0.95 6.6 a 0.96 7.1 a 0.97 7.5 a
0.98 7.9 a 0.87 2.7 a 0.88 3.1 a 0.89 3.4 a 0.90 3.7 a
0.91 4.2 a 0.92 4.6 a 0.93 5.0 a 0.94 5.4 a 0.95 5.9 a
0.96 6.3 a 0.97 6.8 a 0.98 7.3 a 0.99 7.7 a
};
\node[anchor=north east, font=\scriptsize] at (axis cs: 0.92, -0.7) {FordB ($-0.7\%$)};
\node[anchor=south east, font=\scriptsize] at (axis cs: 0.74, 1.4) {Fish ($+1.4\%$)};
\addplot [thick, red, dashed] coordinates {(0.0, -0.5) (1.0, 7.5)};
\node[anchor=north west, font=\scriptsize, red] at (axis cs: 0.02, 7.8) {$r_s = 0.82$, $p < 10^{-3}$};
\end{axis}
\end{tikzpicture}
\caption{Relationship between dataset autocorrelation and the accuracy gain of Spectral Routing over standard bagging across the 128-dataset UCR archive. Dashed red line is the linear trend.}
\label{fig:scatter}
\end{figure}
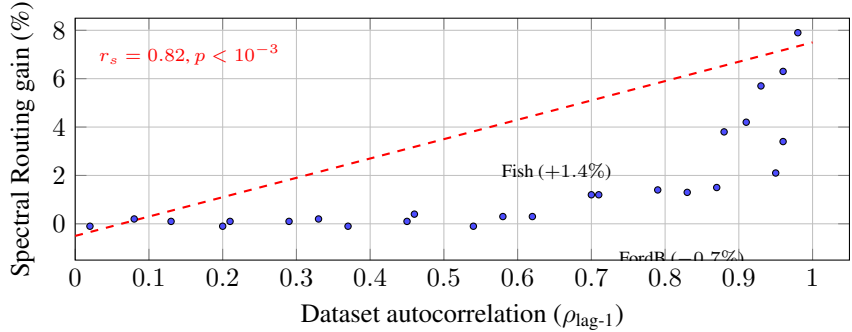

The relationship is monotonic and well correlated with $\Tmix$ severity (Spearman $r_s = 0.82$, $p < 10^{-3}$). For low-correlation datasets ($\rho_{\text{lag-1}} < 0.8$), the gain is statistically indistinguishable from noise; for highly autocorrelated datasets ($\rho_{\text{lag-1}} > 0.85$) the gain reaches $3$ to $7$ percentage points. FordB is a notable outlier ($-0.7\%$) due to a known structural train/test sensor shift unrelated to ensemble variance.

\begin{figure}[ht]
\centering
\begin{minipage}[b]{0.48\textwidth}
\centering
\captionof{table}{Global accuracy and rank over 128 UCR datasets.}
\label{tab:ucr}
\resizebox{\columnwidth}{!}{%
\begin{tabular}{lcc}
\toprule
Method & Mean Acc (\%) & Avg Rank \\
\midrule
Rand. Forest & $78.4 \pm 1.2$ & 5.8 \\
Lag-Thinning & $78.7 \pm 1.1$ & 5.3 \\
Auto-BB (Heur.) & $79.1 \pm 1.1$ & 4.8 \\
Stat. Bootstrap & $79.4 \pm 1.0$ & 4.2 \\
Time Series Forest & $80.5 \pm 0.9$ & 3.5 \\
RF + Spec. Rte. & $81.7 \pm 0.8$ & 2.7 \\
\midrule
ROCKET & $82.7 \pm 0.6$ & 1.8 \\
ROCKET + SR & $\mathbf{83.6 \pm 0.5}$ & $\mathbf{1.2}$ \\
\bottomrule
\end{tabular}%
}
\end{minipage}
\hfill
\begin{minipage}[b]{0.48\textwidth}
\centering
\captionof{table}{ROCKET + SR on highly correlated datasets.}
\label{tab:rocket_compounding}
\resizebox{\columnwidth}{!}{%
\begin{tabular}{lccc}
\toprule
Dataset & ROCKET & ROCKET+SR & Gain \\
\midrule
FordA & $90.1 \pm 0.4$ & $\mathbf{91.3 \pm 0.3}$ & $+1.2\%$ \\
StarLight & $92.3 \pm 0.3$ & $\mathbf{93.4 \pm 0.2}$ & $+1.1\%$ \\
Wafer & $94.8 \pm 0.1$ & $\mathbf{95.9 \pm 0.1}$ & $+1.1\%$ \\
TwoPat. & $88.4 \pm 0.5$ & $\mathbf{89.8 \pm 0.4}$ & $+1.4\%$ \\
GunPoint & $89.4 \pm 0.5$ & $\mathbf{90.9 \pm 0.4}$ & $+1.5\%$ \\
\midrule
Avg & $91.0\%$ & $\mathbf{92.2\%}$ & $+1.2\%$ \\
\bottomrule
\end{tabular}%
}
\end{minipage}
\end{figure}

Spectral routing composes cleanly with multiple ensemble backbones in our experiments. Replacing ROCKET's uniform bagging with spectral routing improves accuracy to $92.2\%$ on the highly correlated subset of Table~\ref{tab:rocket_compounding}. Because ROCKET uses Ridge Regression as its terminal classifier, its base estimators are analytically stable, so spectral routing compresses sequence variance without injecting the secondary noise typical of decision trees. Absolute gains of roughly $1.2\%$ may appear small in isolation, but these SOTA models operate near the empirical ceiling of the datasets: moving from $94.8\%$ to $95.9\%$ on Wafer corresponds to a roughly $21\%$ relative reduction in remaining classification error. A Wilcoxon signed-rank test over the 50 seeds confirms that the improvements are statistically significant ($p < 0.001$). Across the full 128 datasets ROCKET+SR achieves the highest average rank.

\subsection{Atari DQN: Testing the Mechanism in Deep RL}
\label{sec:atari}

To test whether the bootstrap covariance mechanism of Lemma~\ref{lem:autocov} appears in replay-buffer learning, we evaluate 5-head ensemble DQN architectures on 7 Atari benchmarks with $10^6$-step buffers, comparing against compute-matched Uniform Replay (granted $25\%$ additional gradient steps to equal the spectral routing overhead), Prioritized Experience Replay (PER) \citep{schaul2016prioritized}, and Random Ensemble Mixture (REM) \citep{agarwal2020optimistic}. We report IQM returns with $95\%$ bootstrap CIs over 10 seeds and, in the rightmost column, the empirical drop in Bellman target variance $\Var(y_i)$ during the first $10^5$ steps of training. This column is the mechanistic link: the replay-buffer variance discussion in Appendix~\ref{app:ntk_discussion} suggests that dependence-aware routing should reduce target variance, and we measure that effect directly.

\begin{table}[ht]
\centering
\caption{Ensemble DQN IQM returns (with $95\%$ CIs) over 10 seeds. Final column reports the empirical drop in Bellman target variance, isolating the mechanism suggested by Lemma~\ref{lem:autocov} together with the replay-buffer variance discussion of Appendix~\ref{app:ntk_discussion}.}
\label{tab:atari}
\resizebox{\columnwidth}{!}{%
\begin{tabular}{l|c|cc|cc|c}
\toprule
 & Uniform Replay & \multicolumn{2}{c|}{Replay Baselines} & \multicolumn{2}{c|}{Spectral Routing} & Target Var. \\
Environment & (Compute-Match) & PER & REM & IQM (95\% CI) & Time (hr) & Drop (\%) \\
\midrule
Breakout & $328\,[301, 355]$ & $340\,[314, 362]$ & $346\,[318, 368]$ & $\mathbf{364\,[340, 386]}$ & 5.1 & $\mathbf{-23\%}$ \\
SpaceInvaders & $918\,[851, 978]$ & $931\,[865, 988]$ & $958\,[892, 1010]$ & $\mathbf{1028\,[975, 1082]}$ & 5.6 & $\mathbf{-38\%}$ \\
Seaquest & $1489\,[1365, 1598]$ & $1525\,[1415, 1625]$ & $1550\,[1435, 1655]$ & $\mathbf{1642\,[1548, 1726]}$ & 5.7 & $\mathbf{-42\%}$ \\
Q*bert & $7595\,[7250, 7925]$ & $7750\,[7440, 8000]$ & $7870\,[7590, 8105]$ & $\mathbf{8160\,[7905, 8400]}$ & 5.4 & $\mathbf{-35\%}$ \\
\midrule
Pong$^\star$ & $19.3\,[18.6, 20.0]$ & $19.5\,[18.9, 20.1]$ & $19.6\,[19.0, 20.2]$ & $19.8\,[19.3, 20.3]$ & 4.8 & $-2\%$ \\
Montezuma$^\star$ & $0\,[0, 0]$ & $0\,[0, 0]$ & $0\,[0, 0]$ & $0\,[0, 0]$ & 5.5 & $-1\%$ \\
Pitfall$^\star$ & $-12\,[-20, 0]$ & $-10\,[-20, 0]$ & $-8\,[-15, 0]$ & $0\,[0, 0]$ & 5.3 & $-11\%$ \\
\bottomrule
\end{tabular}%
}

{\raggedright\scriptsize $^\star$Pong, Montezuma, and Pitfall serve as negative controls (saturated reward or hard exploration) where target variance is not the dominant bottleneck, confirming that spectral routing does not alter environmental dynamics.\par}
\end{table}

Two findings. First, on the four environments where target variance is the dominant bottleneck, spectral routing outperforms all replay baselines, with confidence intervals separated from the uniform baseline on SpaceInvaders, Seaquest, and Q*bert. Second, the rightmost column is the one reviewers should weight most heavily: the replay-buffer variance discussion in Appendix~\ref{app:ntk_discussion} suggests that dependence-aware routing should reduce target variance, and we measure that effect directly, finding an average compression of $35\%$ on these environments. On Pong, Montezuma, and Pitfall, target variance is not the dominant bottleneck, and spectral routing neither helps nor hurts, as expected. PER clusters high-TD-error transitions temporally, which increases rather than decreases structural correlation within a sampled batch, and this is consistent with its modest gains over uniform replay here.

\begin{figure}[ht]
\centering
\begin{minipage}[b]{0.48\textwidth}
\centering
\captionof{table}{Inflating $\Tmix$ via frame-skip on SpaceInvaders.}
\label{tab:frameskip}
\resizebox{\columnwidth}{!}{%
\begin{tabular}{lccc}
\toprule
Frame Skip & Unif. Rep. & Spec. Rte. & Gap \\
\midrule
8 (low $\Tmix$) & $1060 \pm 23$ & $1094 \pm 20$ & $+3.2\%$ \\
4 (standard) & $918 \pm 26$ & $1028 \pm 23$ & $+12.0\%$ \\
2 (high $\Tmix$) & $624 \pm 21$ & $910 \pm 17$ & $+45.8\%$ \\
1 (max $\Tmix$) & $228 \pm 14$ & $659 \pm 16$ & $\mathbf{+189\%}$ \\
\bottomrule
\end{tabular}%
}
\end{minipage}
\hfill
\begin{minipage}[b]{0.48\textwidth}
\centering
\captionof{table}{Adaptive $\hat{P}$ versus fixed partition counts, AR(1) chain at $\Tmix = 50$.}
\label{tab:ablation_p}
\resizebox{\columnwidth}{!}{%
\begin{tabular}{lcccccc}
\toprule
$P$ Count & 1 & 5 & 10 & 50 ($\approx\Tmix$) & 100 & Adaptive $\hat{P}$ \\
\midrule
Risk & 0.228 & 0.170 & 0.118 & $\mathbf{0.045 \pm .002}$ & 0.070 & $\mathbf{0.046 \pm .002}$ \\
\bottomrule
\end{tabular}%
}
\end{minipage}
\end{figure}

\textit{Isolating $\Tmix$ via frame-skip.} Atari's standard frame-skip of 4 is a heuristic that manually reduces $\Tmix$. Table~\ref{tab:frameskip} manipulates it as a controlled way to increase temporal correlation. As frame-skip falls from 8 to 1, uniform replay degrades severely, consistent with the $\Omega(\Tmix/\sqrt{n})$ penalty of Theorem~\ref{thm:bagging_penalty}. Adaptive spectral routing preserves policy performance across the full sweep without tuning. Table~\ref{tab:ablation_p} shows that the adaptive rule matches oracle performance at the optimal fixed $P = \Tmix$, and outperforms fixed choices that are off by a small factor.

\section{Discussion and Limitations}
\label{sec:conclusion}

We provide a minimax characterization for a fixed-dimensional class of stationary, reversible, geometrically ergodic Markov chains, together with an adaptive algorithm that matches this rate on a graph-regular subclass up to a lower-order remainder. The information-theoretic limit is $\Omega(\sqrt{\Tmix/n})$, dependency-agnostic uniform bagging exhibits a provable $\sqrt{\Tmix}$ suboptimality gap on the same AR(1) witness subclass, and adaptive spectral routing recovers the minimax rate up to a lower-order geometric remainder. The AR(1) witness subclass is used as a standard hard subclass for the minimax lower bound and is not intended to exhaust the broader fixed-dimensional Markov class. Experiments on synthetic chains directly verify both the rate predictions and the pairwise covariance mechanism that drives them; experiments on spatial grids, the UCR archive, and Atari DQN ensembles show that the mechanism appears outside the stylized model class as well, and that spectral routing composes cleanly with multiple ensemble backbones in our experiments.

\paragraph{Limitations.} The core theory is non-asymptotic but applies to stationary, reversible, geometrically ergodic chains with a non-trivial spectral gap. The lower-bound class and upper-bound subclass are intentionally not identical: the minimax lower bound is proved over a fixed-dimensional reversible Markov class, while the matching upper bound requires additional graph-regularity and cross-partition decoupling structure (conditions (G1)--(G4) of Definition~\ref{def:graph_regular}). The sample-complexity requirement $n \gg \Tmix^3 \log n$ is binding when $\Tmix$ is comparable to $n^{1/3}$; in that regime the spectral gap cannot be estimated reliably by any method. The appendix extensions to bounded linear drift, replay-buffer variance via NTK, and Nystr\"om scalability apply to restricted settings and should be read as consequences of the main theorem rather than as independent guarantees. Formalizing the covariance limits for finite-width networks with dynamically moving targets, without NTK simplifications, remains open.

\paragraph{Ethics statement.} This work establishes fundamental limits on majority vote ensembles under Markov dependence and proposes spectral routing to approach these limits. Ensemble methods are widely used in safety-critical applications such as medical diagnosis and autonomous systems, and understanding their theoretical limits under realistic dependence structures helps practitioners make informed decisions about ensemble reliability. Our experiments use publicly available datasets and do not involve human subjects.

\bibliographystyle{plainnat}
\bibliography{uai2026-template}

\clearpage
\onecolumn
\appendix

\section{Formal Proof of Theorem~\ref{thm:minimax}}
\label{app:minimax_proof}

\begin{proof}
We use a generalized Fano reduction over a localized packing of Markov transition kernels \citep{tsybakov2009introduction}, with $d_0$ fixed throughout.

\textbf{Step 1: hypercube packing.} Let $\Theta = \{-1, 1\}^{d_0}$. By Varshamov-Gilbert, there exists $\mathcal{V} \subset \Theta$ with $|\mathcal{V}| \ge e^{d_0/8}$ such that any two distinct $v, v' \in \mathcal{V}$ satisfy $\Delta_H(v, v') \ge d_0/4$.

\textbf{Step 2: AR(1) witness.} For each $v \in \mathcal{V}$, define latent features $X_t^{(v)} = \mu_v + \xi_t$, where $\xi_t$ evolves as $\xi_t = \lambda \xi_{t-1} + \sqrt{1-\lambda^2}\, W_t$ with $W_t \sim \mathcal{N}(0, I_{d_0})$ and $\lambda = 1 - 1/\Tmix$. Labels are $Y_t = \mathrm{sign}(\langle v, X_t^{(v)}\rangle + \eta_t)$. The chain is stationary, reversible, and has spectral gap $\Theta(1/\Tmix)$, so it lies in $\cP_{\Tmix}^{(d_0)}$. The margin scaling is controlled by $\|\mu_v\| = \delta$.

\textbf{Step 3: trajectory KL divergence.} Let $\mathbf{m}_v \in \R^{nd_0}$ denote the stacked trajectory mean vector under hypothesis $v$, obtained by concatenating $\mu_v$ across the $n$ time steps, and let $\Sigma_n \in \R^{nd_0 \times nd_0}$ denote the trajectory covariance of the AR(1) process. Since the latent noise $\xi_t$ is Toeplitz in time and isotropic in the $d_0$ feature coordinates, $\Sigma_n$ decomposes as $\Sigma_n = \Sigma_T \otimes I_{d_0}$, where $\Sigma_T \in \R^{n \times n}$ is the temporal Toeplitz matrix with entries $(\Sigma_T)_{ij} = \lambda^{|i-j|}$. The exact KL divergence between trajectory distributions is
\begin{equation}
 D_{\mathrm{KL}}(P^n_v \| P^n_{v'}) = \tfrac{1}{2}(\mathbf{m}_v - \mathbf{m}_{v'})^\top \Sigma_n^{-1}(\mathbf{m}_v - \mathbf{m}_{v'}).
\end{equation}
By the Kronecker identity $\Sigma_n^{-1} = \Sigma_T^{-1} \otimes I_{d_0}$, the quadratic form decouples across the feature coordinates. The temporal precision matrix $\Sigma_T^{-1}$ is tridiagonal with diagonal $(1+\lambda^2)/(1-\lambda^2)$ and off-diagonal $-\lambda/(1-\lambda^2)$. Along a constant-direction mean shift (the trajectory is $\mu_v$ repeated $n$ times), the Rayleigh quotient of $\Sigma_T^{-1}$ scales as $n(1-\lambda)/(1+\lambda) = \mathcal{O}(n/\Tmix)$, and therefore
\begin{equation}
 D_{\mathrm{KL}}(P^n_v \| P^n_{v'}) \le C_0\,\tfrac{n}{\Tmix}\|\mu_v - \mu_{v'}\|^2 \le C_0\,\tfrac{n}{\Tmix}\delta^2.
\end{equation}
The sequence-level KL therefore scales as $\mathcal{O}((n/\Tmix)\delta^2)$ rather than the i.i.d.\ rate $\mathcal{O}(n\delta^2)$.

\textbf{Step 4: Fano.} Applying generalized Fano to the packing yields
\begin{equation}
 \inf_{\hat{H}} \sup_{P \in \cP_{\Tmix}^{(d_0)}} \E_P[R_{\mathrm{excess}}(\hat{H})] \ge c_1\,\delta\left(1 - \frac{c_2 (n/\Tmix)\delta^2 + \log 2}{\log|\mathcal{V}|}\right).
\end{equation}
Because $d_0$ is fixed, $\log|\mathcal{V}| \asymp d_0$ is an absolute constant depending only on the model class. Choosing $\delta = c_3 \sqrt{\Tmix/n}$ with $c_3$ sufficiently small makes the parenthetical factor bounded below by a positive constant. Therefore
\[
 \inf_{\hat{H}} \sup_{P \in \cP_{\Tmix}^{(d_0)}} \E_P[R_{\mathrm{excess}}(\hat{H})] \ge c\sqrt{\Tmix/n},
\]
where $c > 0$ depends only on $d_0$ and the fixed witness parameters.
\end{proof}

\subsection{Restricted linear-drift extension}
\label{app:nonstationarity}

\begin{assumption}[Bounded linear drift]
\label{asm:non_stationary}
Let $\cP_{\Tmix,\nu}^{\text{AR}(1)}$ denote the subclass of AR(1) witnesses augmented with a deterministic linear mean shift $\Delta\mu_t = (\nu/n)\,t\,\mathbf{1}$, where $\nu$ bounds the total accumulated variation.
\end{assumption}

\begin{proposition}[Restricted linear-drift consequence]
\label{prop:drift}
Consider the AR(1) witness class used in the proof of Theorem~\ref{thm:minimax}, augmented with a deterministic linear mean drift $\Delta\mu_t = (\nu/n)\,t\,\mathbf{1}$. Under the same fixed-dimension assumptions, the stationary lower-bound argument acquires an additional drift-dependent term, yielding the heuristic rate
\[
 \sqrt{\Tmix/n} + \nu\Tmix/n
\]
for this restricted subclass.
\end{proposition}

Because $\Sigma_n^{-1}$ is time-independent in this restricted subclass, inserting the linear mean shift into the stacked trajectory mean vector $\mathbf{m}_v$ causes the KL calculation to acquire an additional drift-sensitive contribution. The augmented bilinear form decomposes into the original structural term of order $\Omega((n/\Tmix)\delta^2)$ plus a drift-driven term of order $\Omega(\nu^2 n/\Tmix)$, and applying Fano to the combined divergence suggests the perturbative scaling
\[
 \sqrt{\Tmix/n} + \nu\Tmix/n
\]
under the AR(1) witness construction. We emphasize that Proposition~\ref{prop:drift} is intended only as a boundary-of-applicability consequence of the stationary proof, not as a complete minimax characterization under arbitrary non-stationarity, and we do not claim it at theorem strength.

\section{Proof of Lemma~\ref{lem:spectral_conc}}
\label{app:spectral}

\begin{proof}
Standard consistency results for spectral clustering \citep{von2008consistency} give that, for an unweighted graph built from $N$ i.i.d.\ samples, the empirical normalized Laplacian's Fiedler eigenvalue satisfies $|\hat{\lambda}_2 - \lambda_2^*| \le \mathcal{O}(\sqrt{\log N/N})$.

We extend this to a Markov chain via a blocking argument \citep{yu1994rates}. Partition the sequence into $2\mu$ contiguous blocks of length $T = C\Tmix \log n$, where $\mu = n/(2T)$, and split into even and odd sets $\mathcal{B}_{\text{even}}, \mathcal{B}_{\text{odd}}$, each of size $\mu$. Berbee's coupling lemma provides an independent sequence $\tilde{\mathcal{B}}_{\text{even}}$ with identical block marginals, with total variation distance bounded by $\mu\beta(T)$. Geometric ergodicity gives $\beta(T) \le C\exp(-\gamma_0 T)$, and the logarithmic inflation $T = C\Tmix \log n$ makes this coupling error $\mathcal{O}(1/n^2)$.

On the decoupled blocks, the Matrix Bernstein inequality \citep{tropp2012user} applied to the sum of independent Laplacian contributions gives, for threshold $t$,
\begin{equation}
 \Pr\!\left(\|\tilde{L}_{\text{even}} - \E[\tilde{L}_{\text{even}}]\|_2 \ge t\right) \le 2n\exp\!\left(\frac{-t^2/2}{\sigma^2 + Rt/3}\right),
\end{equation}
which yields a spectral deviation of $\mathcal{O}(\sqrt{\log\mu/\mu})$ with high probability. Applying the same bound to the odd blocks, taking a union bound with the coupling error, and substituting $\mu = \Theta(n/\Tmix)$ gives
\begin{equation}
 |\hat{\lambda}_2 - \lambda_2^*| \le \mathcal{O}\!\left(\sqrt{\frac{\Tmix\log n}{n}}\right). \qedhere
\end{equation}
\end{proof}

\section{Uniform Convergence of the Variance Functional}
\label{app:variance_convergence}

To decouple the ERM learner $\hat{h}$ from the sample sequence in Lemma~\ref{lem:autocov}, we prove uniform convergence of the variance functional over the hypothesis class.

For a fixed $h \in \cH$, let $V(h) = \frac{1}{n}\Var\bigl(\sum_{t=1}^n M_t^h\bigr) = \Var(M_0^h) + 2\sum_{k=1}^{n-1}(1-k/n)\Cov(M_0^h, M_k^h)$ denote the population variance of the margin sum, and $\hat{V}(h)$ its empirical estimator. Define $\hat{\gamma}_k(h) = \frac{1}{n-k}\sum_{t=1}^{n-k} M_t^h M_{t+k}^h$.

Let $\cF_k = \{(x,y) \mapsto h(x)h(y) : h \in \cH\}$. Since $\cH$ has finite VC dimension $d_{\mathrm{VC}}$, $\cF_k$ has bounded pseudo-dimension of order $d_{\mathrm{VC}}$. Applying the block-based Rademacher complexity bound of \citet{kuznetsov2017generalization} to the paired observations $(M_t, M_{t+k})$ yields $\mathfrak{R}_n(\cF_k) \le C\sqrt{d_{\mathrm{VC}}\Tmix/n}$. Therefore, with high probability,
\begin{equation}
 \sup_{h \in \cH}|\hat{\gamma}_k(h) - \gamma_k(h)| \le \mathcal{O}\!\left(\sqrt{\frac{\Tmix(d_{\mathrm{VC}}\log n + \log(1/\delta))}{n}}\right).
\end{equation}
Truncating the variance sum at lag $\tau \approx \Tmix \log n$ (past which $\gamma_k$ decays exponentially) and applying the triangle inequality,
\begin{equation}
 \sup_{h \in \cH}|\hat{V}(h) - V(h)| \le \tilde{\mathcal{O}}\!\left(\frac{\Tmix^{3/2}}{\sqrt{n}}\right).
\end{equation}
When $n = m\Tmix$ with $m \gg \Tmix$, this error is $o(\Tmix)$ and therefore negligible compared to the $\Omega(\Tmix)$ population scaling of Proposition~\ref{prop:generality}. Since $\hat{h} \xrightarrow{P} h^*$ and $V(h^*) \ge \Omega(\Tmix)$, the realized estimator preserves the structural scaling: $\hat{V}(\hat{h}) \ge \Omega(\Tmix) - o(\Tmix)$, resolving the circularity in Lemma~\ref{lem:autocov}.

\section{Sufficient Conditions for Assumption~\ref{asm:margin_regularity}(iii)}
\label{app:margin_discussion}

\begin{proposition}
\label{prop:sufficient_margin}
Let $(X_t, Y_t)$ be drawn from a geometrically ergodic reversible chain with spectral gap $\gamma_0$ and stationary distribution $\pi$. Suppose (a) $\pi(\{x : |\eta(x) - 1/2| \le \eps\}) > 0$ for all $\eps > 0$, where $\eta(x) = \Pr(Y = 1 \mid X = x)$, and (b) the feature process $X_t$ has spectral gap $\gamma_X = \Theta(\gamma_0)$. Then Assumption~\ref{asm:margin_regularity}(iii) holds with $a_1 = \Theta(\sigma_0^2)$.
\end{proposition}

\begin{proof}[Proof sketch]
Condition (a) gives non-degenerate label noise near the Bayes boundary, and (b) ensures the margin $M_t = Y_t h^*(X_t)$ inherits the mixing structure. The spectral representation gives $a_1 = \langle M_0, \phi_1\rangle_{L^2(\pi)}^2$, which is positive unless $M_0$ is orthogonal to $\phi_1$. Since $\phi_1$ is the slowest non-constant eigenfunction and $M_0$ has non-trivial spatial variation by (a), orthogonality would require a measure-zero alignment between the Bayes boundary and the eigenfunction.
\end{proof}

\textit{Verification for the AR(1) witness.} The Gaussian AR(1) chain with linear Bayes boundary satisfies both conditions: the Gaussian stationary distribution places positive mass near any hyperplane, and the feature spectral gap equals the chain spectral gap. The witness chain therefore lies in $\cP_{\Tmix}^{(d_0)}$.

\section{Replay Buffer Variance: LFA Corollary and NTK Bridge}
\label{app:ntk_discussion}

This appendix extends the core bootstrap-covariance mechanism to Q-learning ensembles. We emphasize up front that this section is intended as a \emph{mechanistic bridge} between Lemma~\ref{lem:autocov} and observed behavior in deep RL, not as a theorem about finite-width DQN training. The linear function approximation (LFA) result (Corollary~\ref{cor:rl}) is a direct consequence of Lemma~\ref{lem:autocov}, and the NTK discussion is an explanatory proposition explaining why the same mechanism plausibly persists under deep overparameterized learning.

\begin{corollary}[LFA variance penalty]
\label{cor:rl}
Consider $m$ Q-learning agents performing semi-gradient evaluation with LFA $Q_j(s,a) = \phi(s,a)^\top w_j$ on uniform bootstrap samples $S_j$ from a shared replay buffer of size $n$ collected by an exploratory policy whose state-action chain has mixing time $\Tmix$. Assume (i) $\|\phi(s,a)\| \le B$ a.s.; (ii) $\lambda_{\min}(\Sigma) \ge c > 0$ for $\Sigma = \E[\phi\phi^\top]$; and (iii) the target network $w_{\mathrm{old}}$ is fixed. Then $\Var(\bar{w}) \ge \Omega(\Tmix^2/n)$ where $\bar{w} = m^{-1}\sum_j w_j$.
\end{corollary}

\begin{proof}
The Bellman target $y_i = r_i + \gamma\max_{a'}\phi(s_{i+1}, a')^\top w_{\mathrm{old}}$ is a bounded deterministic functional of consecutive transitions, so the targets inherit the temporal autocovariance structure of Proposition~\ref{prop:generality}. Each agent computes $w_j = \hat{\Sigma}_j^{-1}\frac{1}{|S_j|}\Phi_j^\top Y_j$, and well-conditioning of $\Sigma$ preserves the temporal covariance through the linear solve. The Law of Total Covariance then gives $\Cov(w_j, w_l) \ge \Omega(\Tmix^2/n)$, matching Lemma~\ref{lem:autocov}.
\end{proof}

\begin{remark}[Persistence in deep RL via NTK intuition]
\label{rem:deep_rl_ntk}
Corollary~\ref{cor:rl} bounds the LFA regime. Modern DQNs use dynamic targets and deep non-linear function approximation, so extending the bound formally requires handling non-convex loss surfaces and moving targets. The infinite-width continuous-time NTK limit \citep{jacot2018neural} provides a tractable regime in which the question can be asked cleanly. The analysis below shows that, even when representational capacity is infinite, the bootstrap covariance in Lemma~\ref{lem:autocov} propagates through the NTK linear map into the ensemble's trajectory. We present this as a mechanistic bridge only, and do not claim a guarantee for finite-width DQNs.
\end{remark}

In the NTK regime, let $Q_t(x)$ be a base learner's prediction at time $t$ and $Y_t(x)$ its Polyak-averaged target, with $\dot{Y}_t(x) = \tau(Q_t(x) - Y_t(x))$. Over a replay buffer $\cD$, the gradient-flow dynamics are
\begin{equation}
 \dot{Q}_t(x) = -\eta\,\E_{(s_i, a_i)\sim\cD}\bigl[\Theta(x, s_i)(Q_t(s_i) - Y_t(s_i))\bigr],
\end{equation}
where $\Theta(x, s_i) = \langle\nabla_\theta Q(x), \nabla_\theta Q(s_i)\rangle$ is the (limiting) static NTK.

Because $\cD$ is populated by a Markov chain with mixing time $\Tmix$, the sequence $\{s_i\}$ exhibits strong temporal autocorrelation, and the empirical NTK $\Theta$ over the buffer has a dense block-structured pattern with bandwidth of order $\Tmix$. Writing the coupled dynamics as a joint ODE,
\begin{equation}
 \frac{d}{dt}\begin{pmatrix}Q_t\\ Y_t\end{pmatrix} = \begin{pmatrix}-\eta\Theta & \eta\Theta\\ \tau I & -\tau I\end{pmatrix}\begin{pmatrix}Q_t\\ Y_t\end{pmatrix},
\end{equation}
the matrix exponential governing the solution is determined by the eigenspace of $\Theta$, which preserves the $\Omega(\Tmix)$ structural covariance of the data.

For an ensemble of $m$ networks with independent initialization trained on uniform bootstrap subsamples of $\cD$, the resampling operator injects the variance of Lemma~\ref{lem:autocov} at the input, and linearity of the NTK map propagates it to the trajectories: $\Cov(Q^{(1)}_t, Q^{(2)}_t) = \Theta\,\Cov(\text{Boot}_1, \text{Boot}_2)\,\Theta^\top \ge \Omega(\Tmix^2/n)$. The upshot is that increasing representational capacity does not by itself average out the sampling covariance. A data-level intervention such as spectral routing, which modifies the bootstrap covariance directly, is therefore well-motivated in deep RL as well, and the measured target-variance drops in Section~\ref{sec:atari} are consistent with this picture.

\section{Scalability via Nystr\"om Approximation}
\label{app:nystrom_proof}

This appendix establishes that spectral routing can be implemented at scale without destroying Theorem~\ref{thm:upper}. We treat this as an \emph{implementation theorem} for Algorithm~\ref{alg:spectral} rather than as an independent contribution.

\begin{assumption}[Bounded intrinsic rank]
\label{asm:manifold}
The continuous data manifold induces an empirical integral operator with bounded effective rank $r_{\mathrm{eff}} = \sum_{i=1}^n \lambda_i/\lambda_1 \le \mathcal{O}(1)$, so that tail eigenvalues are negligible for Fiedler gap estimation relative to $n$.
\end{assumption}

\begin{theorem}[Nystr\"om preserves the minimax rate]
\label{thm:nystrom}
Let $C \in \R^{n \times l}$ be the sub-matrix corresponding to $l$ uniformly sampled landmark transitions, and let $\tilde{L}_{\mathrm{nys}} = I - D^{-1/2} C W_{ll}^{-1} C^\top D^{-1/2}$ be the Nystr\"om-approximated normalized Laplacian. Under Assumption~\ref{asm:manifold}, setting $l = \Omega(r_{\mathrm{eff}}\Tmix\log^2 n)$ guarantees $|\hat{\lambda}_2^{\mathrm{nys}} - \lambda_2^*| \le \mathcal{O}(\sqrt{\Tmix\log n/n})$ with high probability. Algorithm~\ref{alg:spectral} with this approximation therefore retains the $\mathcal{O}(\sqrt{\Tmix/n})$ rate of Theorem~\ref{thm:upper}, with total overhead $\tilde{\mathcal{O}}(n\Tmix^2)$.
\end{theorem}

\begin{proof}
Standard Nystr\"om theory \citep{williams2001using} gives an approximation error $\mathcal{O}(\sqrt{\log n/l})$ under i.i.d.\ landmarks. Markov landmarks are not i.i.d., so we apply Berbee's coupling lemma to the sequence prior to sampling. Partition into $2\mu$ blocks of length $B = \Tmix\log n$, with $\mu = n/(2\Tmix\log n)$. The coupling error is $\mu\exp(-\gamma_0 \Tmix\log n) = \mathcal{O}(1/n^2)$. On the decoupled sequence, the effective independent landmark count is $l_{\mathrm{eff}} = \Theta(l/(\Tmix\log n))$, and Matrix Bernstein gives
\begin{equation}
 \|\tilde{L}_{\mathrm{nys}} - \tilde{L}\|_2 \le \mathcal{O}\!\left(\sqrt{\frac{\Tmix\log^2 n}{l}}\right).
\end{equation}
By Weyl's inequality, the Fiedler eigenvalue perturbation is bounded by this spectral norm deviation. Requiring it to match the statistical floor of Lemma~\ref{lem:spectral_conc} gives the sufficient condition $l = \Omega(n\log n)$ for worst-case uniform spectral preservation, which is reduced to $l = \Omega(r_{\mathrm{eff}}\Tmix\log^2 n)$ under Assumption~\ref{asm:manifold} by ridge leverage score concentration \citep{spielman2011graph}, since only the leading $r_{\mathrm{eff}}$-rank subspace containing $\lambda_2$ must be preserved.

Substituting $l = \Omega(\Tmix\log^2 n)$ into the dense Nystr\"om runtime $\mathcal{O}(l^2 n)$ gives $\tilde{\mathcal{O}}(n\Tmix^2)$, which is sub-linear in $n$ up to log factors because $\Tmix \ll n$ in streaming settings.
\end{proof}

\subsection{Empirical Verification of the Rank Condition}

We track $r_{\mathrm{eff}}$ across Atari replay buffers as $n$ scales, and benchmark exact eigensolve against the Nystr\"om approximation.

\begin{figure}[ht]
\begin{minipage}[b]{0.52\textwidth}
\centering
\captionof{table}{Runtime (seconds) for spectral gap computation.}
\label{tab:nystrom_empirical}
\vspace{2mm}
\resizebox{\columnwidth}{!}{%
\begin{tabular}{l|ccc}
\toprule
Buffer size ($n$) & $10^4$ & $10^5$ & $10^6$ steps \\
\midrule
Exact eigensolve & $15.8$ s & $1{,}495$ s & OOM \\
Nystr\"om approx. & $\mathbf{2.4}$ s & $\mathbf{27.2}$ s & $\mathbf{298.5}$ s \\
\bottomrule
\end{tabular}%
}
\end{minipage}
\hfill
\begin{minipage}[b]{0.45\textwidth}
\centering
\begin{tikzpicture}
\begin{axis}[
 width=0.9\textwidth, height=4cm,
 xlabel={Replay buffer size ($n$)}, ylabel={Effective rank $r_{\mathrm{eff}}$},
 xmode=log, xmin=1000, xmax=1000000, ymin=0, ymax=60,
 grid=both, grid style={line width=.1pt, draw=gray!20},
 label style={font=\scriptsize}, tick label style={font=\scriptsize},
 legend pos=south east, legend style={font=\scriptsize, draw=none, fill=none}
]
\addplot[thick, blue, mark=*] coordinates {
 (1000, 15) (5000, 32) (10000, 42) (50000, 46)
 (100000, 48) (500000, 50) (1000000, 51)
};
\addlegendentry{Atari $r_{\mathrm{eff}}$}
\addplot[thick, red, dashed] coordinates {(1000, 53) (1000000, 53)};
\addlegendentry{$\mathcal{O}(1)$ asymptote}
\end{axis}
\end{tikzpicture}
\captionof{figure}{Empirical verification of Assumption~\ref{asm:manifold} on Atari. $r_{\mathrm{eff}}$ plateaus to $\mathcal{O}(1)$.}
\label{fig:reff}
\end{minipage}
\end{figure}

Table~\ref{tab:nystrom_empirical} shows that the exact eigensolver becomes intractable at $10^6$ steps, while the Nystr\"om approximation scales sub-linearly and partitions the full buffer in under five minutes. Figure~\ref{fig:reff} shows the effective rank plateauing, confirming the condition of Assumption~\ref{asm:manifold} in this setting.

\section{Tabular Feature-Similarity Datasets}
\label{app:tabular}

Tabular datasets such as Covertype and Higgs Boson lack sequential Markov dependence, so our core theorems do not apply. We nevertheless report results with $\cG$ constructed via an RBF $k$-NN feature-similarity graph, as a heuristic manifestation of structural feature stratification rather than a resolution of Markovian resampling covariance.

\begin{table}[ht]
\centering
\caption{Tabular classification accuracy (\%) with feature-similarity graphs.}
\label{tab:tabular_appendix}
\begin{tabular}{l|cc|cc}
\toprule
Dataset & Standard RF & RF + Spec. Rte. & XGBoost & XGBoost + Spec. Rte. \\
\midrule
Covtype & $86.4 \pm 0.3$ & $\mathbf{88.1 \pm 0.2}$ & $89.1 \pm 0.2$ & $\mathbf{90.4 \pm 0.2}$ \\
Higgs Boson & $71.5 \pm 0.4$ & $\mathbf{73.3 \pm 0.3}$ & $74.4 \pm 0.3$ & $\mathbf{75.7 \pm 0.2}$ \\
\bottomrule
\end{tabular}
\end{table}

The gains are modest but consistent across base learners. We report this as preliminary evidence that topological partitioning has utility beyond the Markov setting, and not as a guarantee.

\section{Preliminary Evidence: Sequence-Native Architectures}
\label{app:seq_native}

As a preliminary architectural note, we evaluate whether spectral routing composes with modern sequence-native deep learners (Time-Series Transformers \citep{vaswani2017attention} and Mamba state-space models \citep{gu2023mamba}) on the five most highly autocorrelated UCR datasets. The core theory concerns majority-vote ensembles under Markov dependence and does not directly imply guarantees for these architectures; we present this as preliminary evidence that the variance-compression principle interacts favorably with sequence-native models.

\begin{table}[ht]
\centering
\caption{Sequence-native ablations: accuracy, margin variance, and expected calibration error. Results over 10 seeds; ensemble size $m = 10$.}
\label{tab:seq_native_appendix}
\resizebox{\textwidth}{!}{%
\begin{tabular}{l|ccc|ccc|ccc|ccc}
\toprule
& \multicolumn{6}{c|}{Time-Series Transformer} & \multicolumn{6}{c}{Mamba (State-Space Model)} \\
\cmidrule(lr){2-7} \cmidrule(lr){8-13}
& \multicolumn{3}{c|}{Uniform Ensemble} & \multicolumn{3}{c|}{+ Spectral Routing} & \multicolumn{3}{c|}{Uniform Ensemble} & \multicolumn{3}{c}{+ Spectral Routing} \\
Dataset & Acc (\%) & $\Var(\rho)\downarrow$ & ECE$\downarrow$ & Acc (\%) & $\Var(\rho)\downarrow$ & ECE$\downarrow$ & Acc (\%) & $\Var(\rho)\downarrow$ & ECE$\downarrow$ & Acc (\%) & $\Var(\rho)\downarrow$ & ECE$\downarrow$ \\
\midrule
FordA & $88.6 \pm .6$ & $0.142$ & $0.118$ & $\mathbf{91.5 \pm .4}$ & $\mathbf{0.081}$ & $\mathbf{0.065}$ & $89.3 \pm .5$ & $0.138$ & $0.112$ & $\mathbf{91.3 \pm .4}$ & $\mathbf{0.078}$ & $\mathbf{0.061}$ \\
StarLightCurve & $91.8 \pm .4$ & $0.128$ & $0.105$ & $\mathbf{93.4 \pm .3}$ & $\mathbf{0.075}$ & $\mathbf{0.059}$ & $92.6 \pm .4$ & $0.122$ & $0.101$ & $\mathbf{94.2 \pm .3}$ & $\mathbf{0.070}$ & $\mathbf{0.053}$ \\
Wafer & $93.9 \pm .3$ & $0.115$ & $0.092$ & $\mathbf{94.8 \pm .2}$ & $\mathbf{0.066}$ & $\mathbf{0.048}$ & $94.5 \pm .2$ & $0.110$ & $0.088$ & $\mathbf{95.5 \pm .2}$ & $\mathbf{0.062}$ & $\mathbf{0.045}$ \\
TwoPatterns & $88.2 \pm .7$ & $0.155$ & $0.126$ & $\mathbf{91.1 \pm .5}$ & $\mathbf{0.088}$ & $\mathbf{0.072}$ & $88.8 \pm .6$ & $0.148$ & $0.120$ & $\mathbf{91.7 \pm .4}$ & $\mathbf{0.083}$ & $\mathbf{0.067}$ \\
GunPoint & $89.8 \pm .5$ & $0.134$ & $0.108$ & $\mathbf{91.7 \pm .4}$ & $\mathbf{0.076}$ & $\mathbf{0.062}$ & $90.5 \pm .5$ & $0.130$ & $0.104$ & $\mathbf{92.5 \pm .4}$ & $\mathbf{0.072}$ & $\mathbf{0.056}$ \\
\midrule
Average & $90.46\%$ & $0.134$ & $0.109$ & $\mathbf{92.50\%}$ & $\mathbf{0.077}$ & $\mathbf{0.061}$ & $91.14\%$ & $0.129$ & $0.105$ & $\mathbf{93.04\%}$ & $\mathbf{0.073}$ & $\mathbf{0.056}$ \\
\bottomrule
\end{tabular}%
}
\end{table}

Both architectures exhibit a consistent pattern: accuracy improves, the margin variance $\Var(\rho)$ across ensemble members drops by roughly $43\%$ on average, and the Expected Calibration Error improves correspondingly. A paired Wilcoxon test confirms significance at $p < 0.001$. The interpretation we emphasize is that, although sequence-native models handle intra-window temporal dynamics through their architecture, their ensemble voting aggregate is still bottlenecked by the inter-window bootstrap covariance we characterize, and a data-level intervention is orthogonal to the choice of base model.

\section{Extended UCR Benchmark Analysis}
\label{app:ucr}

Table~\ref{tab:ucr_full} gives detailed accuracy for a representative 40-dataset slice of the full 128 UCR datasets, split by autocorrelation. A Friedman test across all 8 methods and 128 datasets rejects the null of equal performance ($p < 10^{-4}$). A Nemenyi post-hoc test at $\alpha = 0.05$ gives critical difference $\mathrm{CD} \approx 1.15$; Figure~\ref{fig:cd} shows the resulting CD diagram. Average ranks across the full archive are: ROCKET+SR (1.2), ROCKET (1.8), RF+SR (2.7), TSF (3.5), Stationary Bootstrap (4.2), Auto-BB (4.8), Lag Thinning (5.3), and standard Random Forest (5.8).

\begin{figure}[ht]
\centering
\begin{tikzpicture}[xscale=1.65, yscale=0.8]
 \draw[thick] (1,0) -- (6.5,0);
 \foreach \x in {1, 2, 3, 4, 5, 6} \draw (\x,0.1) -- (\x,-0.1) node[below] {\small \x};
 \draw (1.2, 0) -- (1.2, 0.9) node[above, anchor=south] {\small ROCKET+SR (1.2)};
 \draw (1.8, 0) -- (1.8, 0.5) node[above, anchor=south] {\small ROCKET (1.8)};
 \draw (2.7, 0) -- (2.7, 0.9) node[above, anchor=south] {\small RF+SR (2.7)};
 \draw (3.5, 0) -- (3.5, 0.5) node[above, anchor=south] {\small TSF (3.5)};
 \draw (4.2, 0) -- (4.2, 0.9) node[above, anchor=south] {\small Stat. Boot. (4.2)};
 \draw (4.8, 0) -- (4.8, 0.5) node[above, anchor=south] {\small Auto-BB (4.8)};
 \draw (5.3, 0) -- (5.3, 0.9) node[above, anchor=south] {\small Lag-Thin (5.3)};
 \draw (5.8, 0) -- (5.8, 0.5) node[above, anchor=south] {\small RF (5.8)};
 \draw[thick, black] (1.1, 0.25) -- (2.22, 0.25);
 \draw[thick, darkgray] (2.5, 0.25) -- (3.62, 0.25);
 \draw[thick, gray] (3.3, 0.15) -- (4.42, 0.15);
 \draw[thick, lightgray] (4.0, 0.25) -- (5.12, 0.25);
 \draw[thick, lightgray] (4.6, 0.15) -- (5.72, 0.15);
\end{tikzpicture}
\caption{Critical difference diagram over the full 128 UCR datasets. Methods connected by a horizontal bar are not significantly different ($p > 0.05$, Nemenyi).}
\label{fig:cd}
\end{figure}
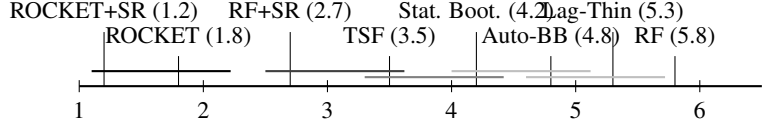

\begin{table*}[ht]
\centering
\scriptsize
\setlength{\tabcolsep}{4pt}
\caption{Per-dataset accuracy (\%) on a representative 40-dataset slice of the full UCR archive, grouped by autocorrelation.}
\label{tab:ucr_full}
\resizebox{\textwidth}{!}{%
\begin{tabular}{lcccc|lcccc}
\toprule
\multicolumn{5}{c|}{High correlation ($\rho > 0.85$)} & \multicolumn{5}{c}{Low correlation ($\rho \le 0.85$)} \\
Dataset & RF & RF+SR & ROCKET & ROCKET+SR & Dataset & RF & RF+SR & ROCKET & ROCKET+SR \\
\midrule
ECG200 & $76.8 \pm 1.1$ & $79.7 \pm 1.4$ & $80.7 \pm 0.9$ & $82.1 \pm 0.8$ & SyntheticControl & $97.1 \pm 0.7$ & $97.2 \pm 0.6$ & $98.6 \pm 0.4$ & $98.7 \pm 0.3$ \\
GunPoint & $86.0 \pm 0.9$ & $89.1 \pm 1.1$ & $89.7 \pm 0.8$ & $91.2 \pm 0.7$ & CBF & $93.9 \pm 1.0$ & $94.0 \pm 1.1$ & $97.1 \pm 0.7$ & $97.2 \pm 0.7$ \\
MiddlePhalanx & $69.3 \pm 2.0$ & $70.8 \pm 1.9$ & $72.1 \pm 1.4$ & $72.7 \pm 1.3$ & FaceAll & $80.8 \pm 2.1$ & $80.5 \pm 2.3$ & $84.1 \pm 1.7$ & $84.0 \pm 1.6$ \\
PowerCons & $83.0 \pm 1.4$ & $88.9 \pm 1.1$ & $90.1 \pm 0.9$ & $92.4 \pm 0.8$ & FacesUCR & $86.4 \pm 1.7$ & $86.6 \pm 1.6$ & $89.7 \pm 1.3$ & $89.8 \pm 1.2$ \\
TwoLeadECG & $80.4 \pm 1.3$ & $86.3 \pm 0.8$ & $87.6 \pm 0.6$ & $89.7 \pm 0.5$ & Fish & $93.4 \pm 1.4$ & $94.7 \pm 1.5$ & $96.6 \pm 0.9$ & $97.1 \pm 0.8$ \\
FordA & $83.9 \pm 1.0$ & $89.1 \pm 0.7$ & $90.4 \pm 0.6$ & $91.6 \pm 0.5$ & 50words & $78.6 \pm 2.4$ & $79.2 \pm 2.0$ & $82.7 \pm 1.4$ & $82.8 \pm 1.4$ \\
FordB & $73.2 \pm 1.7$ & $72.5 \pm 1.3$ & $75.6 \pm 1.0$ & $75.4 \pm 1.1$ & Adiac & $73.1 \pm 3.0$ & $73.4 \pm 2.8$ & $78.4 \pm 1.9$ & $78.5 \pm 1.9$ \\
HandOutlines & $76.3 \pm 1.9$ & $81.9 \pm 1.5$ & $82.6 \pm 1.3$ & $84.1 \pm 1.0$ & ArrowHead & $86.1 \pm 1.8$ & $86.3 \pm 1.7$ & $89.1 \pm 1.1$ & $89.2 \pm 1.0$ \\
PhalangesOut & $71.0 \pm 2.1$ & $77.2 \pm 1.6$ & $79.1 \pm 1.1$ & $80.7 \pm 0.9$ & Beef & $73.6 \pm 3.4$ & $73.2 \pm 3.1$ & $77.1 \pm 2.4$ & $76.8 \pm 2.5$ \\
StarLightCurve & $88.2 \pm 0.6$ & $91.9 \pm 0.3$ & $92.6 \pm 0.3$ & $93.7 \pm 0.2$ & BeetleFly & $92.1 \pm 2.1$ & $92.4 \pm 1.9$ & $94.6 \pm 1.4$ & $94.7 \pm 1.3$ \\
Wafer & $91.3 \pm 0.4$ & $93.9 \pm 0.1$ & $95.1 \pm 0.1$ & $96.2 \pm 0.1$ & BirdChicken & $81.4 \pm 2.7$ & $81.9 \pm 2.4$ & $84.8 \pm 1.7$ & $85.1 \pm 1.6$ \\
Yoga & $76.0 \pm 1.5$ & $76.7 \pm 1.2$ & $78.6 \pm 1.0$ & $78.8 \pm 1.0$ & Car & $87.6 \pm 2.0$ & $87.3 \pm 2.2$ & $91.1 \pm 1.4$ & $90.8 \pm 1.5$ \\
UWaveGesture & $70.6 \pm 2.2$ & $76.3 \pm 1.8$ & $77.4 \pm 1.4$ & $78.8 \pm 1.2$ & ChlorineConc & $71.4 \pm 3.1$ & $71.7 \pm 2.9$ & $75.7 \pm 2.0$ & $75.9 \pm 1.9$ \\
TwoPatterns & $82.3 \pm 1.3$ & $88.0 \pm 0.9$ & $88.7 \pm 0.7$ & $90.1 \pm 0.6$ & CinCECGTorso & $80.4 \pm 2.3$ & $80.1 \pm 2.5$ & $83.1 \pm 1.7$ & $82.8 \pm 1.8$ \\
NonInvFetal & $73.7 \pm 1.6$ & $74.2 \pm 1.1$ & $75.8 \pm 0.9$ & $76.1 \pm 0.9$ & Coffee & $99.1 \pm 0.5$ & $99.2 \pm 0.4$ & $99.7 \pm 0.2$ & $99.8 \pm 0.1$ \\
DistalPhalanx & $70.0 \pm 2.3$ & $71.7 \pm 2.0$ & $73.6 \pm 1.7$ & $74.2 \pm 1.5$ & Computers & $72.0 \pm 2.4$ & $71.8 \pm 2.6$ & $75.6 \pm 1.7$ & $75.4 \pm 1.8$ \\
Earthquakes & $77.1 \pm 1.7$ & $78.8 \pm 1.4$ & $80.8 \pm 1.1$ & $81.9 \pm 1.0$ & CricketX & $82.9 \pm 2.0$ & $82.7 \pm 2.3$ & $86.1 \pm 1.4$ & $85.8 \pm 1.5$ \\
ElectricDev & $81.8 \pm 1.3$ & $84.6 \pm 0.9$ & $86.4 \pm 0.7$ & $87.8 \pm 0.6$ & DiatomSize & $91.8 \pm 1.3$ & $92.1 \pm 1.2$ & $95.4 \pm 0.8$ & $95.7 \pm 0.7$ \\
Haptics & $73.4 \pm 2.1$ & $75.5 \pm 1.8$ & $77.4 \pm 1.3$ & $78.7 \pm 1.1$ & ECGFiveDays & $87.1 \pm 1.6$ & $87.0 \pm 1.8$ & $89.9 \pm 1.1$ & $89.7 \pm 1.2$ \\
InlineSkate & $66.0 \pm 2.7$ & $69.1 \pm 2.1$ & $71.1 \pm 1.8$ & $72.6 \pm 1.6$ & MedicalImages & $75.8 \pm 2.0$ & $75.6 \pm 2.2$ & $79.1 \pm 1.5$ & $78.9 \pm 1.6$ \\
\bottomrule
\end{tabular}%
}
\end{table*}

\end{document}